\documentclass[letterpaper]{article} 
\usepackage[submission]{aaai23}  
\usepackage{times}  
\usepackage{helvet}  
\usepackage{courier}  
\usepackage[hyphens]{url}  
\usepackage{graphicx} 
\usepackage{caption} 
\usepackage{algorithm}
\usepackage{algorithmic}

\usepackage{newfloat}
\usepackage{listings}
\DeclareCaptionStyle{ruled}{labelfont=normalfont,labelsep=colon,strut=off} 
\floatstyle{ruled}
\newfloat{listing}{tb}{lst}{}
\floatname{listing}{Listing}
\usepackage{commath}
\usepackage{amsmath}
\usepackage{amssymb}
\usepackage{amsthm}
\usepackage{bm}
\usepackage{bbm}
\usepackage{color}
\usepackage{enumitem}
\usepackage{subfigure}
\usepackage{cases}
\usepackage{aecompl}

\newtheorem{fact}{Fact}
\newtheorem{definition}{Definition}
\newtheorem{theorem}{Theorem}
\newtheorem{lemma}{Lemma}

\title{Regret Analysis for Hierarchical Experts Bandit Problem}
\author {
    Qihan Guo,
    Siwei Wang,
    Jun Zhu
}
\affiliations {
    Department of Computer Science and Technology, Tsinghua University, Beijing 100084, China\\
    \texttt{qihan.guo@163.com, acmilan120791899@126.com, dcszj@mail.tsinghua.edu.cn}
}

\usepackage{bibentry}

\begin{document}

\maketitle

\begin{abstract}
    We study an extension of standard bandit problem in which there are $R$ layers of experts. Multi-layered experts make selections layer by layer and only the experts in the last layer can play arms. The goal of the learning policy is to minimize the total regret in this hierarchical experts setting. We first analyze the case that total regret grows linearly with the number of layers. Then we focus on the case that all experts are playing Upper Confidence Bound (UCB) strategy and give several sub-linear upper bounds for different circumstances. Finally, we design some experiments to help the regret analysis for the general case of hierarchical UCB structure and show the practical significance of our theoretical results. This article gives many insights about reasonable hierarchical decision structure.
\end{abstract}

\section{Introduction}

This paper focuses on stochastic multi-armed bandit problem, a basic online decision problem. In a multi-armed bandit problem, an algorithm must choose from one of the $K$ possible actions in each of the $n$ consecutive rounds. At each round, the chosen action brings some observable payoff following an unknown prior distribution. The goal of the algorithm is to maximize the total payoff (i.e., the sum of the payoffs of the chosen actions in each round) by choosing arms properly based on previous observations. The performance of an algorithm in bandit problem is typically measured in terms of \textit{regret}, which is defined as the difference between the expected obtained payoff of the algorithm and the obtained payoff of an optimal strategy, i.e., always selects the action with the highest expected payoff over the $n$ rounds.

Bandit problems have many application scenarios. For instance, in clinical trials \cite{clinical}, one may have several possible treatments for a certain class of individuals but does not know which one is the best. In terms of bandit problem, these treatments are possible actions and the performance of patients is observable payoff. One needs an algorithm to choose treatments at different time steps to minimize the sacrifice of patients. Another instance is ad placement problem. Ad placement \cite{sum} is the problem of deciding what advertisement to display on the web page delivered to the visitors with particular feature. Similarly, every advertisement is an available action and the number of user clicks is the payoff. The goal of the ad placement algorithm is to maximize the number of total user clicks.

In real applications, sometimes the system follows a hierarchical decision structure. For instance, when a company wants to invest in stocks, it will try to follow one of the investment institutions. These investment institutions do not make decisions directly, and they will follow some of their respective employees. Finally, groups of employees will choose specific stocks directly. All these form a hierarchical decision structure, where the investment company can only decide the selected investment institutions and only employees can directly choose stocks.

Another instance is game tree search. Game tree can be used to model many non-deterministic and imperfect information games such as backgammon \cite{backgamon}, poker \cite{poker} and Scrabble \cite{Scrabble}. Its structure is a tree, where each node of the tree represents a state in the game, and the children of one node represent the next feasible states of the node. The result of each game is not known until the end, i.e., one receives a reward only after a path is selected from the root node to a leaf node. The goal of a game tree search algorithm is to make the winning rate of the game as high as possible by selecting the path properly. Because every node in the tree can only choose from its children, a game tree search algorithm is a hierarchical decision algorithm.

Therefore, in this paper, we consider an expansion of standard bandit problem, which is called \textit{hierarchical experts bandit problem}. In our setting, the controller cannot directly choose the available actions at each round, but only an ``expert'' in the first layer. Similarly, this expert cannot directly choose the available actions as well, he needs to choose an expert in the second layer. In this setting, only the experts in the last layer can directly choose actions. Besides, every expert can only see the selected times and accumulated rewards of the experts in the next layer, and cannot see the situation after the next layer. See Section \ref{sc:pre} for the formal discription.

In the investment company example mentioned above, we can regard investment institutions as the first layer of experts, employees as the second layer of experts, and stocks as the actions. Then it can be modeled as a hierarchical experts bandit problem with two layers of experts. Similarly, every layer in a game tree forms one layer of experts and the game tree can also be modeled as a hierarchical experts bandit problem. 

In hierarchical experts setting, a straightforward analysis is to regard all the experts as special actions. Then the regret in each layer depends on the used algorithms, and the total regret can be bounded by the sum of regrets from each layer. This leads to a regret bound which increases linearly with the number of layers. In our opinion, this bound is not tight in many cases. Therefore, our main goal is improving this bound under some reasonable assumptions about the hierarchical experts structure. In this paper, we will analyze that under what conditions the total regret grows linearly with the number of layers, and under what conditions the total regret can be bounded by a sub-linear upper bound. From these analysis, we can get some inspiration about what a reasonable hierarchical decision structure should be like.

The structure of this paper is as follows. Section \ref{sc:rel} describes the related work of this paper and the comparison between these work and ours. Section \ref{sc:pre} gives the formal definition of the problem we study and some basic facts which are useful in following sections. Section \ref{sc:lb} describes a circumstance that total regret increases linearly with the number of layers, implying that we need strong assumptions about the experts in hierarchical structure to restrict the total regret. Section \ref{sc:ub} gives some circumstances that total regret grows sub-linearly with the number of layers, which is the main theoretical result of this paper. Section \ref{sc:exp} describes some experiments which demonstrates the regret analysis in Section \ref{sc:ub}, and shows the practical significance of them in more general problem instances. Section \ref{sc:con} gives the conclusion of this paper and lists some of our future research interests.

\section{Related Work} \label{sc:rel}

The multi-armed bandit model was first introduced by \cite{rob}. After that, \cite{lai} proved the lower bound of regret of algorithms in stochastic multi-armed bandit problem is $\Omega(\log n)$. Then, \cite{auer} introduced the algorithm Upper Confidence Bound (UCB) and showed that the optimal regret bound of order $O(\log n)$ can be achieved uniformly over time for the stochastic bandit problem. Some improvement of UCB such as KL-UCB \cite{kl-ucb} further reduced the distance between the upper and lower bound of regret by using KL-divergence to improve the confidence bound in UCB. There are some other asymptotically optimal strategies such as Thompson Sampling \cite{thompson} and $\varepsilon$-Greedy \cite{auer}. Thompson Sampling chooses arms according to a posterior distribution of uniform distribution. $\varepsilon$-Greedy chooses the arm with highest sample mean with propability $1-\varepsilon$ and chooses arms according to a uniform distribution with propability $\varepsilon$ at each time step.

Multi-armed bandit problem also has adversarial version. In adversarial setting, the reward of an arm does not follow a fixed distribution. There is an adversary who knows the strategy of the controller in advance. At each time step, the reward of an arm follows a distribution setted by the adversary who tries to minimize the total reward obtained by the controller. For this setting, \cite{exp3} proposed the algorithm Exp3 which achieves the regret bound of order $O(\sqrt{n K \log K})$, leaving a $\sqrt{\log K}$ factor gap from the lower bound of order $\Omega(\sqrt{n K})$ proposed in \cite{exp3}.

\cite{uct} proposed the algorithm UCT, which deals with a similar structure to hierarchical experts setting. UCT is an extension of UCB to minimax tree search. In the tree structure, every node is an independent bandit, and the algorithm must play a sequence from the root to a leaf at each round. The idea of UCT is to regard child-nodes as independent arms and use UCB at each node. UCT has asymptotically regret $O(K \log n)$, where $K$ is the number of leaves. Its difference with our hierarchical experts setting is that it is an overall strategy that can decide the selection of every node, and in our hierarchical experts setting the controller can only select the experts in the first layer and later selection depends on the algorithms in following layers. Besides, every node in the tree is just an arm following some distribution, not an expert with some strategy. These differences make the analysis in two settings very different, and one cannot apply the results in \cite{uct} to ours.

The algorithm Exp4 proposed in \cite{exp3} also considers to choose from experts in bandit problem. It can also deal with adversarial bandit problem and achieves the regret bound of order $O(\sqrt{n N \log K})$, where $N$ is the number of experts. NEXP \cite{nexp}, an improvement of Exp4, achieves the regret bound of order $O(\sqrt{n S \log N})$, where $S \leq \min\{K,N\}$. However, both algorithms just refer to experts' advice and can \textbf{directly} play arms, which is an important difference from hierarchical experts setting. Also, because Exp4 considers the adversarial version, it can only get regret bound $O(\sqrt{n})$, while our hierarchical experts setting is a stochastic setting and can achieve regret bound of $O(\log n)$.

\section{Preliminaries} \label{sc:pre}

\subsection{Model Setting}
In hierarchical experts bandit problem, there is a fixed pool of actions (also called \textit{arm}) $\{1,2,\cdots,K\}$, with $K$ known to all experts. Successive plays of arm $i$ yield rewards $X_{i, 1}, X_{i, 2}, \ldots \in [0,1]$ which are independent and identically distributed according to an unknown law with unknown expectation $\mu_i$. Without loss of generality, we assume that $\mu_1 > \mu_2 \geq \cdots \geq \mu_K$ and $K \geq 2$. Independence also holds for rewards across arms, i.e., $X_{i,s}$ and $X_{j,t}$ are independent (and usually not identically distributed) for each $1 \leq i < j \leq K$ and each $s, t \geq 1$.

An \textit{expert} $A$ is an algorithm that chooses the next expert/arm to play based on the sequence of past plays and obtained rewards. There are $R$ layers, and layer $k$ consists of $L_k$ experts $a_{1}^{k}, a_{2}^{k}, \ldots, a_{L_k}^{k}$. Suppose we have a controller $B$ (We will call it ``top strategy'' in the following). At time $t$, $B$ selects an expert $a^1(t)$ from layer $1$ based on past plays on layer $1$ and obtained rewards, then $a^1(t)$ selects an expert $a^2(t)$ from layer $2$, $\cdots$, until $a^R(t)$ plays an arm $I_t$ and get a reward $X(t)$. The reward of $a^k(t)$ at time $t$ is defined as $X^k(t) = X(t)$, $\forall 1 \leq k \leq R$.

There is an important environment assumption: Each expert in layer $k$ can see and can only see the value of $a^{k+1}(t)$ and $X^{k+1}(t)$, $\forall 0 \leq k \leq R, 1 \leq t \leq n$. In other words, at time $t$, $a^k(t)$ must make decision based on the values of $\{a^{k+1}(1), \cdots, a^{k+1}(t-1)\}$ and $\{X^{k+1}(1), \cdots, X^{k+1}(t-1)\}$. Here layer $0$ refers to $B$ and layer $R+1$ refers to the arm set. This environment assumption corresponds to the information-missing situation in reality. For example, when a company finds an investment institution to invest in stocks, the company cannot know the specific stock choices of the employees in the institution, only the overall performance of this investment institution.

The \textit{total regret} is defined as
\begin{equation}
R_{n}^{*} = \max _{i=1, \ldots, K} \mathbb{E}\left[\sum_{t=1}^{n} X_{i,t}-\sum_{t=1}^{n} X(t)\right]
\label{regret_top}
\end{equation}

Here the expectation is taken with respect to the random draw of both rewards and expert's actions (some experts may have internal randomization).

\subsection{Basic Facts}
We firstly introduce Upper Confidence Bound (UCB) strategy, a basic and important strategy in stochastic bandit problem. The idea of UCB is to construct an upper bound estimate on the mean of each arm's payoff at some fixed confidence level and then choose the arm that looks best under this estimate. Note that in this section, the \textit{regret} $R_n$ of an algorithm is defined as the total regret of a hierarchical experts structure where top strategy is the algorithm and $R=0$.

\begin{algorithm}[H]
    \caption{$\alpha$-UCB}
    \label{algo:UCB}
    \begin{algorithmic}
        \STATE \hspace{-5mm} \textbf{Input Parameters:} $\alpha > 2$
        \STATE Intialize $S_i(0) = 0$ and $T_i(0) = 0$ for $i \in \{1, \ldots, K\}$
        \FOR {$t\in \{1, \ldots, n\}$}
        \IF {$1 \leq t \leq K$}
        \STATE Play arm $I_t = t$
        \ELSE
        \STATE Play arm $I_{t} = \underset{i=1, \ldots, K}{\operatorname{argmax}}\left(\frac{S_i(t-1)}{T_i(t-1)} + \sqrt{\frac{\alpha \ln t}{2 T_i(t-1)}}\right)$
        \ENDIF
        \STATE $T_{I_t}(t) = T_{I_t}(t-1) + 1$
        \STATE Observe reward $X(t)$
        \STATE $S_{I_t}(t) = S_{I_t}(t-1) + X(t)$
        \ENDFOR
    \end{algorithmic}
\end{algorithm}

\begin{fact} (\textbf{Regret of $\alpha$-UCB, Theorem 1 in \cite{auer}}) \label{fact:ucb}
    $\alpha$-UCB satisfies
    \begin{equation}
        R_{n} \leq \sum_{i=2}^{K}\left(\frac{2 \alpha}{\Delta_{i}} \ln n + \frac{\alpha}{\alpha-2}\right), \quad \Delta_{i} = \mu_{1} - \mu_i \nonumber
    \end{equation}
\end{fact}

The following result then shows that UCB is asymptotically optimal.

\begin{fact} (\textbf{Distribution-dependent lower bound, Theorem 2 in \cite{lai}}) \label{fact:lb}
    Consider a strategy that satisfies $\mathbb{E}\left[T_{i}(n)\right]=o\left(n^{a}\right)$ for any set of Bernoulli reward distributions, any arm $i > 1$, and any $a > 0$. Then, for any set of Bernoulli reward distributions, the following holds:
    \begin{equation}
        \liminf _{n \rightarrow+\infty} \frac{R_{n}}{\ln n} \geq \sum_{i=2}^{K} \frac{\Delta_{i}}{\mathrm{kl}\left(\mu_{i}, \mu_{1}\right)} \nonumber
    \end{equation}
    Here $\mathrm{kl}(p, q)=p \ln \frac{p}{q}+(1-p) \ln \frac{1-p}{1-q}$.
\end{fact}

Finally we introduce Hoeffding's inequality \cite{hoeffding}, an important inequality which forms the basis of analysis in bandit problems, and will be used frequently in our analysis.

\begin{fact} (\textbf{Hoeffding's inequality}) \label{fact:hoe}
    Let $X_1, \cdots, X_n$ be independent random variables such that $a_{i} \leq X_{i} \leq b_{i}$ almost surely. Define $S_n = X_1 + \cdots + X_n$, then $\forall \varepsilon > 0$,
    \begin{equation}
        \begin{aligned}
        \mathbb{P}\left(S_{n}-\mathbb{E}\left[S_{n}\right] \geq \varepsilon\right) \leq \exp \left(-\frac{2 \varepsilon^{2}}{\sum_{i=1}^{n}\left(b_{i}-a_{i}\right)^{2}}\right) \\
        \mathbb{P}\left(S_{n}-\mathbb{E}\left[S_{n}\right] \leq -\varepsilon\right) \leq \exp \left(-\frac{2 \varepsilon^{2}}{\sum_{i=1}^{n}\left(b_{i}-a_{i}\right)^{2}}\right)
        \end{aligned}
    \end{equation}
\end{fact}

\section{Regret Lower Bound} \label{sc:lb}

In this section, we will introduce a negative result, i.e., in some specific circumstances, the total regret increases linearly with the number of layers. To state this result formally, we firstly give some definitions.

\begin{definition}
    We say an expert $A$ is \textbf{reasonable} if it satisfies: For the selection range (arms/experts) $\{a_1,\cdots,a_k\}$ of $A$, if $\exists 1 \leq i \leq k$ s.t. $\mathbb{E}\left[a_{i,t}\right] = \underset{1 \leq j \leq k}{\min} \mathbb{E}[a_{j,t}], \forall t > 0$, then $\exists$ constant $C$ independent of $n$ s.t. the expectation of the number of times $A$ chooses $a_i$ is no more than $\frac{n}{k} + C$ if $A$ is chosen by $n > 1$ times. Here $a_{i,t}$ denotes the reward of $a_i$ when $a_i$ is selected for the $t$-th time.
\end{definition}

\begin{definition}
    We say an expert $A$ is \textbf{stable} if it satisfies: For the selection range (arms/experts) $\{a_1,\cdots,a_k\}$ of $A$, if $\exists \mu_1 > \mu_2$, $1 \leq i \leq k$ s.t. $\underset{1 \leq j \leq k}{\max} \mathbb{E}[a_{j,t}] \leq \mu_1, \forall t > 0$ and $\mathbb{E}[a_{i,t}] \geq \mu_2, \forall t > 0$, then $\exists$ constant $C$ independent of $n$ s.t. the number of times $a_i$ is chosen (not necessarily by $A$, maybe other experts in the same layer as $A$) is no less than $\frac{1}{\mathrm{kl}\left(\mu_{2}, \mu_{1}\right)} \ln n - C$ if $A$ is chosen by $n > 1$ times.
\end{definition}

There are many classic algorithms in stochastic bandit problem satisfying these two definitions, such as UCB and $\varepsilon$-Greedy. Then we can introduce the result:

\begin{theorem}
    \label{thm:bad}
    Suppose $\forall 1 \leq k \leq R, L_k \geq 3$. Layer $k$ has $L_k-2$ reasonable and stable experts and $2$ ``bad'' experts, $\forall 1 \leq k \leq R$ (``bad'' expert means we can arbitrarily decide its strategy). Top strategy $B$ is reasonable and stable. Then for any arm set, there is an implement of ``bad'' experts s.t. $$R_n^{*} = \Omega(R \log n)$$
\end{theorem}

The detailed proof is given in Appendix \ref{app:proof}. From Theorem \ref{thm:bad}, we know that if we want the total regret to be upper bounded by $o(R\log n)$, we better have restrictions on each expert's strategy. Therefore, in Section \ref{sc:ub}, we focus on the case that all the experts are optimal and stable, and show that the regret upper bound can be sub-linear with the number of layers in this case.

\section{Sub-linear Regret Upper Bound} \label{sc:ub}

We call a hierarchical experts structure \textit{hierarchical UCB structure} if every expert (including top strategy) is playing UCB policy. In this section, we analyze the total regret of hierarchical UCB structure.

There are some reasons we focus on hierarchical UCB structure. First, UCB is the most widely-used stragtegy in stochastic bandit problems because of its simple structure and asymptotical optimality (shown by Fact \ref{fact:ucb} and Fact \ref{fact:lb}). Second, UCB is a good representation of realistic strategies, i.e., the parameter $\alpha$ in UCB has a clear meaning in reality: the propensity to explore. The larger $\alpha$ is, the more the algorithm tends to explore the arms with fewer pulls. Third, UCB has good inductive properties, which will be explained in detail in Theorem \ref{thm:ucb_jg}.

We firstly introduce a circumstance where hierarchical experts structure only leads to a constant increase of regret. It shows the inductive properties of UCB and gives important inspiration for our following theorems. The technique used in the proof is learned from the proof of Theorem 2.1 in \cite{sum}.

\begin{theorem} \label{thm:ucb_jg}
    Suppose $R =1$ and $a_1^1$ is $\alpha$-UCB. Let $J_t$ denote the index of the expert selected by top strategy at time $t$. $\forall x \leq y, C(x,y) := \sum_{t=x}^{y} \mathbb{I}_{J_{t} \neq 1}$, which equals to the number of times top strategy does \textbf{not} selects $a_1^1$ between time $x$ and time $y$. For fixed $n$, suppose top strategy satisfies: $ C(t,n) \leq \frac{2 \alpha \ln n}{\Delta_K^2} - \frac{2 \alpha \ln t}{\Delta_K^2}, \forall 1 \leq t \leq n$. Then $\exists$ constant $C_{\alpha}$ (related to $\alpha$) s.t.
    $$R_{n}^{*} \leq \sum_{i=2}^{K}\left(\frac{2 \alpha}{\Delta_{i}} \ln n \right) + C_{\alpha}$$
    which is the same as the regret bound of $\alpha$-UCB except for the constant term.
\end{theorem}

The detailed proof is given in Appendix \ref{app:proof}. Note that in Theorem \ref{thm:ucb_jg}, $a_i^1$ with $i > 1$ can be ``bad'' (defined in Theorem \ref{thm:bad}). However, the regret bound in Theorem \ref{thm:ucb_jg} is independent of layer number. This is due to the restriction of $C(t,n)$ in Theorem \ref{thm:ucb_jg}. We can easily validate that $C(t,n)$ does not have this bound in Theorem \ref{thm:bad}.

When UCB is selecting arms, the number of times it does not select the best arm between time $x$ and time $y$ has a similar bound to the assumption in Theorem \ref{thm:ucb_jg} (which will be explained in detail in Lemma \ref{lm:ucb_ul} in Appendix \ref{app:proof}). It inspires us to regard experts as arms when analyzing hierarchical UCB structure, and then use the optimality of UCB to estimate the number of times that the UCB with the smallest parameter at each layer is selected, so as to finally obtain the estimation of total regret.

In the following, we give two regret bounds which grow sub-linearly with the number of layers, suitable for two different kinds of hierarchical UCB structure. The first bound only uses the parameters of UCB at the bottom layer:

\begin{theorem} \label{thm:good_bottom}
    Suppose there are $R$ layers of experts. $a_j^R$ is $\alpha_j$-UCB, $\forall 1 \leq j \leq L_R$. Let $\alpha^* = \max \{\alpha_1, \cdots, \alpha_{L_R}\}$, then $\exists$ constant $C_{\alpha^*}$ s.t.
    $$R_{n}^{*} \leq \sum_{i=2}^{K}\left(\frac{2 \alpha^*}{\Delta_{i}} \ln n\right)+C_{\alpha^*}$$
\end{theorem}

\begin{proof}
Recall that $a^k(t)$ denotes the expert selected at layer $k$ at time $t$ and $T_i(m)$ denote the number of times arm $i$ is selected during the first $m$ rounds. Let $J_t$ denote the index of $a^R(t)$, from the proof of Theorem \ref{thm:ucb_jg} we have $\exists \text{ constant } C_{\alpha^*}, \forall 1 \leq i \leq K, \forall 1 \leq j \leq L_R$,
\begin{align}
\mathbb{E}\left[\sum_{t=1}^{n} \mathbb{I}_{I_{t}=i, J_{t}=j}\right] \leq &C_{\alpha^*} + \nonumber\\
&\mathbb{E}\left[\sum_{t=1}^{n} \mathbb{I}_{T_{i}(t-1)<\frac{2 \alpha_{j} \ln t}{\Delta_{i}^{2}}, I_{t}=i, J_{t}=j}\right] \nonumber
\end{align}
So we have
\begin{align}
    \mathbb{E}\left[T_i(n)\right] \leq& \sum_{j=1}^{L_R} \mathbb{E}\left[\sum_{t=1}^{n} \mathbb{I}_{I_{t}=i, J_{t}=j}\right] \nonumber\\
    \leq& L_R C_{\alpha^*} + \nonumber\\
    &\sum_{j=1}^{L_R} \mathbb{E}\left[\sum_{t=1}^{n} \mathbb{I}_{T_{i}(t-1)<\frac{2 \alpha_{j} \ln t}{\Delta_{i}^{2}}, I_{t}=i, J_{t}=j}\right] \nonumber\\
    \leq& L_R C_{\alpha^*} + \mathbb{E}\left[\sum_{t=1}^{n} \mathbb{I}_{T_{i}(t-1)<\frac{2 \alpha^* \ln t}{\Delta_{i}^{2}}, I_{t}=i}\right] \nonumber\\
    \leq& \frac{2 \alpha^* \ln n}{\Delta_{i}^{2}} + L_R C_{\alpha^*} + 1 \nonumber
\end{align}
\end{proof}

Theorem \ref{thm:good_bottom} tells us that if we fix the experts of the bottom layer, adding the number of layers will not make the total regret bound worse. So this bound is suitable for the case that experts at the bottom layer all have small parameters.

However, when $L_R$ is large, there is a great possibility that $\alpha^*$ is also large, which makes the regret bound in Theorem \ref{thm:good_bottom} bad. In this case, we need another kind of regret bound.

\begin{theorem} \label{thm:final}
    Suppose there are $R$ layers of experts with $1<L_1<L_2<\cdots<L_R<K$. $a_j^k$ is $\alpha_j^k$-UCB (Here the $k$ in $\alpha_j^k$ is an index, not the $k$ power of $\alpha_j$), $\forall 1 \leq j \leq L_k, 1 \leq k \leq R$. Top strategy is $\beta$-UCB. Define $\tilde{R}_{n}^{*}= \left(\underset{1 \leq i \leq K}{\max} \sum_{t=1}^{n} X_{i, t}\right)-\sum_{t=1}^{n} X(t)$, i.e., $\mathbb{E}[\tilde{R}_{n}^{*}] = R_{n}^{*}$. \\
    Then the following conclusion holds: $\forall \varepsilon > 0$, $\exists$ constant $M_{\varepsilon}$, if $\alpha_j^k > M_{\varepsilon}, \forall 2 \leq j \leq L_k, 1 \leq k \leq R$ (i.e., only one expert is good in each layer), then $\forall \delta > 0, \exists$ constant $C_{\varepsilon, \delta}$ s.t. with probability $1-\delta$, $\forall n > 0$ we have \\
    If $i^{*} > 2$,
    $$\tilde{R}_n^* \leq \left(\frac{\alpha_1^R}{2 \Delta_{i^{*}-1}^2} \left(\sum_{l=i^{*}}^{K} \Delta_l\right) + \sum_{i=2}^{i^{*}-1} \frac{\alpha_1^R}{2 \Delta_i} + \varepsilon \right) \ln n + C_{\varepsilon, \delta}$$
    Else,
    $$\tilde{R}_{n}^{*} \leq\left(\frac{\underset{k \in \mathcal{S}_{2}}{\sum}(L_k-1)\alpha_1^{k-1}}{2\Delta_{2}^2/\Delta_{K}} +\sum_{i=2}^{K} \frac{\alpha_{1}^R}{2 \Delta_{i}}+\varepsilon\right) \ln n + C_{\varepsilon, \delta}$$
    where
    $$i^* = \min \left\{2 \leq i \leq K: (K-i)\frac{\alpha_1^R}{\Delta_{i}^{2}} - \sum_{l=i+1}^{K}\frac{\alpha_1^R}{\Delta_{l}^{2}} \leq \right.$$
    $$\left.\frac{\sum_{k \in \mathcal{S}_i}(L_k-1)\alpha_1^{k-1}}{\Delta_i^2} \right\}, \quad \alpha_1^{0} = \beta$$
    $$\mathcal{S}_m = \left\{1 \leq k \leq R: (L_k-1)\frac{\alpha_1^{k-1}}{\Delta_m^2} \geq \max\{\right.$$
    $$\left.(L_1-1)\frac{\beta}{\Delta_K^2}, (L_2-1)\frac{\alpha_1^1}{\Delta_K^2}, \cdots, (L_{k-1}-1)\frac{\alpha_1^{k-2}}{\Delta_K^2} \} \right\}$$
\end{theorem}

~

The proof of Theorem \ref{thm:final} is a little complicated and we refer readers to Appendix \ref{app:proof} for its complete proof.

In Theorem \ref{thm:final}, $i^* \geq 2$ satisfies: when $n$ is large, arm $i < i^*$ will only be chosen by $a_1^R$. If $i^* > 2$, from the optimality of $a_1^R$, $T_{i^*-1}(n) \leq O(\frac{1}{\Delta_{i^*-1}} \ln n)$. From the property of UCB, $T_{i}(n)$ is no more than $T_{i^*-1}(n)$ plus a constant when $i \geq i^*$. So the total regret can be bounded by $O(\frac{1}{\Delta_{i^*-1}} \ln n)$. Therefore, the larger $i^*$ is, the better the bound will be.

We can see that when there are less layers or more arms, $i^*$ will become larger, and therefore lead to a lower regret upper bound. Besides, larger parameters at each layer do not influence the bound, which is also an advantage of the bound in Theorem \ref{thm:final}.

Overall, if the parameters in the bottom layer are small, or the layer number is large, Theorem \ref{thm:good_bottom} will give a good bound. If the layer number is small, or the arm set is large, Theorem \ref{thm:final} will give a good bound. In next section, we will see that the bound in Theorem \ref{thm:final} is also a good bound in general case.

\section{Experiments} \label{sc:exp}

In Theorem \ref{thm:final} we require some UCB algorithms have very large parameters. In this secton, we try to remove this requirement and use some experiments to help understand the regret of general hierarchical UCB structure. Also, we will use experiments to show the practical significance of our theoretical results. The implementation details of the experiments can be found in Appendix \ref{app:exp}.

\subsection{Regret change after increasing parameters of UCB strategies} \label{exp1}
We firstly show that in a hierarchical UCB structure, if at each layer all the non-minimal parameters increase to very large (i.e., making the hierarchical UCB structure satisfy the conditions of Theorem \ref{thm:final}), then the total regret will almost certainly become larger. This implies that the regret upper bound in Theorem \ref{thm:final} is also a good bound for hierarchical UCB structure in general, even without the constraint that some parameters are very large.

In this experiment, we consider four kinds of reward distributions: deterministic, Bernoulli, Beta and Binomial (normalized to $[0,1]$). Here ``deterministic'' means the arm returns the same value each time. In each case, we randomly generate the arm set and hierarchical UCB structure. By fixing the total number of rounds and repeatly running processes, we estimate the expectation of the total regret in this setting and denote it $R_1$. Then at each layer, we maintain the minimal parameter and increase other parameters to a very large number. Similarly, we estimate the expectation of the total regret in this setting, denote it $R_2$.

For every kind of distribution, we run independently 1000 processes to see the relationships between $R_1$ and $R_2$. The experimental results are as follows:

\begin{table}[H]
    \centering
    \scalebox{0.82}{
    \begin{tabular}{|c|c|c|c|c|}
        \hline
        distribution & Deter & Bernoulli & Beta & Binomial \\
        \hline
        times of $R_2 \geq R_1$ & 988 & 982 & 981 & 978 \\
        \hline
        times of $R_2 < R_1$ & 12 & 18 & 19 & 22 \\
        \hline
        proportion of $R_2 < R_1$ & 1.2\% & 1.8\% & 1.9\% & 2.2\% \\
        \hline
    \end{tabular}}
    \caption{Comparison of $R_1$ and $R_2$} \label{tb:R_12}
\end{table}

\begin{figure}[hbt]
    \centering
    \subfigure[]{\includegraphics[width=.22\textwidth]{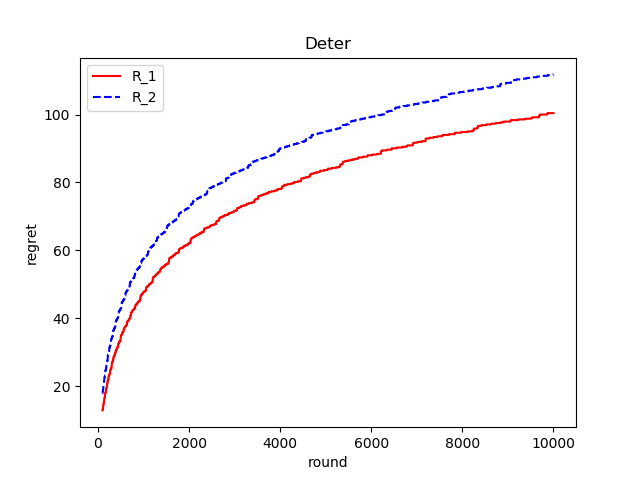}\label{fig:a}}
    \subfigure[]{\includegraphics[width=.22\textwidth]{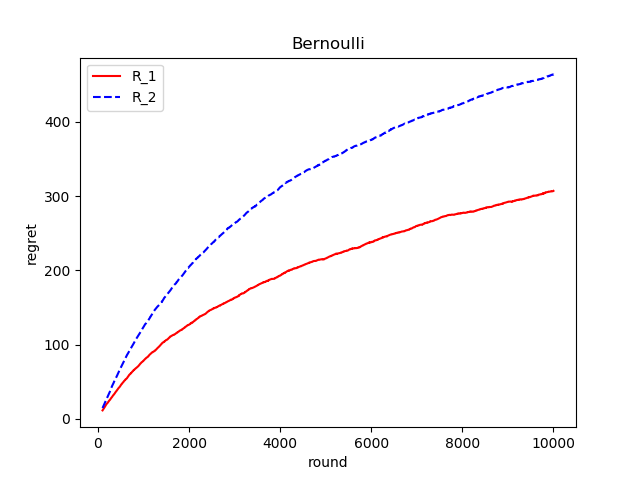}\label{fig:b}}
    \subfigure[]{\includegraphics[width=.22\textwidth]{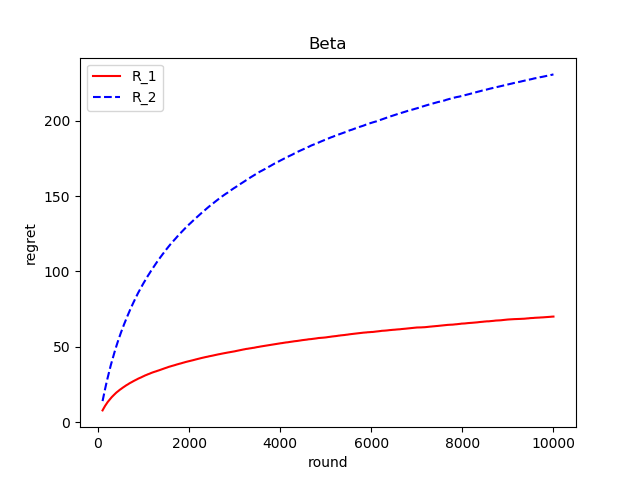}\label{fig:c}}
    \subfigure[]{\includegraphics[width=.22\textwidth]{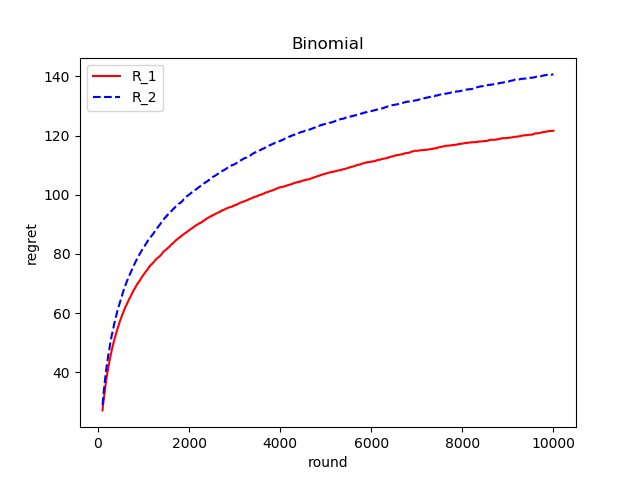}\label{fig:d}}
    \subfigure[]{\includegraphics[width=.22\textwidth]{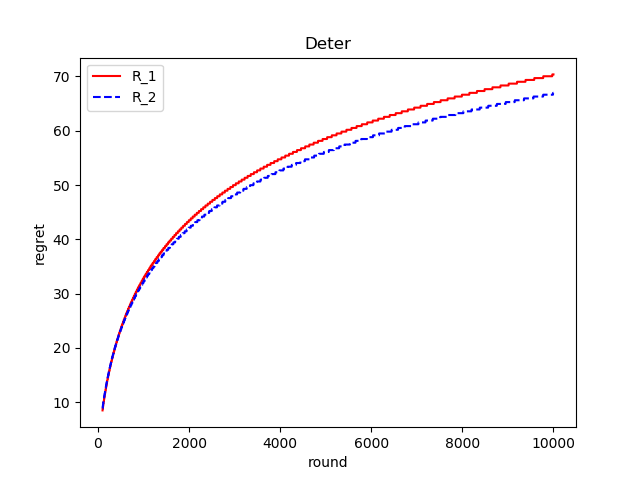}\label{fig:e}}
    \subfigure[]{\includegraphics[width=.22\textwidth]{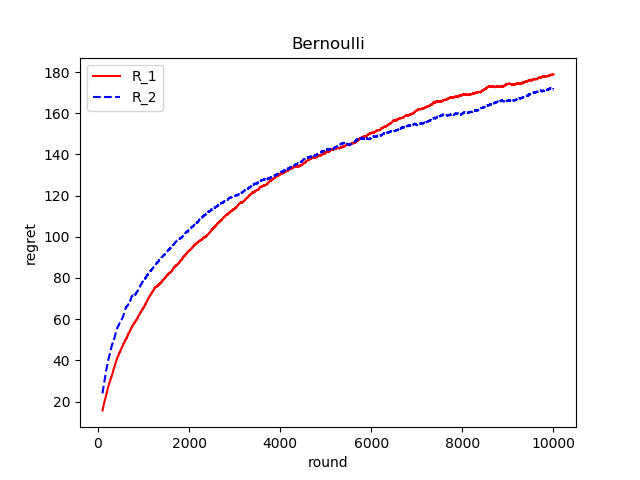}\label{fig:f}}
    \caption{Regret curves of $R_1$ and $R_2$}
\end{figure}

From Table \ref{tb:R_12}, we can see that in most cases we have $R_2 \geq R_1$, implying in most cases Theorem \ref{thm:final} actually gives an upper bound (e.g., as shown in Figure \ref{fig:a} \ref{fig:b} \ref{fig:c} \ref{fig:d}). However, there are still limited cases that we cannot use Theorem \ref{thm:final} to bound the total regret (e.g., as shown in Figure \ref{fig:e} \ref{fig:f}). In these cases, we cannot only consider the effect of each layer's minimal parameter on the total regret and need a more detailed upper bound.

Nevertheless, from the experiments, we can also see that even when $R_1 > R_2$, the gap between them is very small (compared with the cases that $R_1 < R_2$). How to identify these cases and what is the regret upper bound in these cases is one of our future research topics.

\subsection{Selection ranges of different experts in hierarchical UCB structure} \label{subsc:exp2}
In the proof of Theorem \ref{thm:final}, we find that every expert has this property: There are some arms/experts which the expert pulls for no more than constant times. However, the UCB policy in classic bandit problem does not have this property (see detail in Lemma \ref{lm:ucb_ul} in Appendix \ref{app:proof}). Therefore, we also wonder that whether this phenomenon appears in a general UCB hierarchical structure (without the constraints that some parameters are very large).

We ran some hierarchical UCB processes to answer this question and the answer is YES. We show an example under deterministic reward distribution with $R=2, n=10000$. The experimental result is recorded in Table \ref{tb:selection} and the parameters of this example are recorded in Table \ref{tb:exp2}. Note that we have put the experts at each layer in order: In this example, at each layer, expert with smaller index has smaller parameter.

\begin{table}[H]
    \centering
    \begin{tabular}{|c|c|c|c|c|c|}
        \hline
         & $a_1^2$ & $a_2^2$ & $a_3^2$ & $a_4^2$ & $a_5^2$ \\
        \hline
        $a_1^1$ & \textbf{5155} & \textbf{1772} & 9 & 0 & 2 \\
        \hline
        $a_2^1$ & 4 & \textbf{1458} & \textbf{1482} & 3 & 1 \\
        \hline
        $a_3^1$ & 3 & 8 & \textbf{12} & \textbf{33} & \textbf{11} \\
        \hline
        $a_4^1$ & 1 & 4 & 3 & \textbf{14} & \textbf{25} \\
        \hline
    \end{tabular}
\end{table}

\begin{table}[H]
    \centering
    \scalebox{0.82}{
    \begin{tabular}{|c|c|c|c|c|c|c|c|}
        \hline
         & arm $1$ & arm $2$ & arm $3$ & arm $4$ & arm $5$ & arm $6$ & arm $7$ \\
        \hline
        $a_1^2$ & \textbf{4984} & 178 & 0 & 0 & 0 & 1 & 0 \\
        \hline
        $a_2^2$ & \textbf{534} & \textbf{2657} & 25 & 16 & 1 & 7 & 2 \\
        \hline
        $a_3^2$ & 18 & \textbf{1338} & \textbf{81} & \textbf{51} & 11 & 6 & 1 \\
        \hline
        $a_4^2$ & 0 & 1 & 5 & 3 & \textbf{16} & \textbf{19} & 6 \\
        \hline
        $a_5^2$ & 0 & 0 & 1 & 1 & \textbf{11} & \textbf{6} & \textbf{20} \\
        \hline
    \end{tabular}}
    \caption{Record of experts' selection. Each row represents how many times the expert at the front of the row has selected each expert/arm for.} \label{tb:selection}
\end{table}

\begin{table}[H]
    \centering
    \scalebox{0.82}{
    \begin{tabular}{|c|c|c|c|c|c|c|c|}
        \hline
        Value of $i$ & 1 & 2 & 3 & 4 & 5 & 6 & 7 \\
        \hline
        $\beta$ & 5.75 & & & & & & \\
        \hline
        $\alpha_i^1$ & 4.04 & 5.33 & 7.24 & 8.32 & & & \\
        \hline
        $\alpha_i^2$ & 2.33 & 5.22 & 5.27 & 7.29 & 8.41 & & \\
        \hline
        $\mu_i$ & 0.94 & 0.93 & 0.54 & 0.42 & 0.21 & 0.20 & 0.06 \\
        \hline
    \end{tabular}}
    \caption{Parameters of the example} \label{tb:exp2}
\end{table}

To see the result more clearly, we ran $\alpha_1^2$-UCB, $\alpha_2^2$-UCB, $\alpha_3^2$-UCB separately on the same arm set, where the total rounds of $\alpha_i^2$-UCB is same as the number of times $a_i^2$ is selected in the example. We use $b_i^2$ to denote these UCB strategies and show the comparison of their selection results and the selection results of $a_i^2$ in Table \ref{tb:b_i}.

\begin{table}[H]
    \centering
    \scalebox{0.82}{
    \begin{tabular}{|c|c|c|c|c|c|c|c|}
        \hline
         & arm $1$ & arm $2$ & arm $3$ & arm $4$ & arm $5$ & arm $6$ & arm $7$ \\
        \hline
        $a_1^2$ & 4984 & 178 & 0 & 0 & 0 & 1 & 0 \\
        \hline
        $b_1^2$ & 2915 & 2125 & 48 & 30 & 17 & 16 & 12 \\
        \hline
        $a_2^2$ & 534 & 2657 & 25 & 16 & 1 & 7 & 2 \\
        \hline
        $b_2^2$ & 1641 & 1386 & 81 & 53 & 30 & 29 & 22 \\
        \hline
        $a_3^2$ & 18 & 1338 & 81 & 51 & 11 & 6 & 1 \\
        \hline
        $b_3^2$ & 707 & 629 & 61 & 42 & 25 & 24 & 18 \\
        \hline
    \end{tabular}}
    \caption{Selection results of $a_i^2$ and $b_i^2$} \label{tb:b_i}
\end{table}

From these tables, we can see that every expert in the example has a concentration area of selection. For instance, $a_2^2$ mainly selected arm $2$, which is different from its selection when running alone. And the number of times $a_1^2$ selected arm $1$ exceeds a lot than the number when it ran alone. In general, at each layer, experts with smaller indices tend to select experts/arms with smaller indices and experts with larger indices tend to select experts/arms with larger indices. This also consists with the proof of Theorem \ref{thm:final}. It is an interesting phenomenon because it tells us that when a group of experts are making decisions together, good experts will have better performance than usual and bad experts will perform worse.

Overall, different experts in a hierarchical UCB structure has different selection ranges. How to compute these ranges precisely plays an important role in the regret analysis of general hierarchical UCB structure. It is also a direction of our future work.

\subsection{Reasonable number of experts}
In this experiment, we consider the case that $R=1$ and the arm set is known. Our question is: If the parameters of hierarchical UCB structure is randomly drawn from some fixed distribution, which value of $L_1$ minimizes the expected total regret?

This question has some practical significance. For example, if a company needs to recruit some employees to help investing in stocks and the company has some information about the stocks, then the company can use the answer to this question to calculate it should recruit how many employees. For the case that $R \geq 2$, this question is also useful because we can regard the experts in layer $2$ as arms and use the answer to this question to estimate the best value of $L_1$.

The first way to answer this question is using Theorem \ref{thm:final}. For each value of $L_1$, we can calculate the expectation of $\alpha_1^1$ and $\beta$, and then obtain the expected upper bound (ignoring the constant term). When $L_1 = i$, we denote this expected bound as $M_i$. Define $P_1 = \underset{i \geq 1}{\operatorname{argmin}}$ $M_i$. Then $P_1$ is one answer of above question.

The second way to answer the question is using simulation experiment. We can randomly run some processes for each value of $L_1$ and record the average total regret. When $L_1 = i$, we denote this average regret as $N_i$. Define $P_2 = \underset{i \geq 1}{\operatorname{argmin}}$ $ N_i$. Then $P_2$ is also one answer of above question. If we repeat the random processes enough times for each value of $L_1$, we can think $P_2$ close to the correct answer.

Therefore, if for different arm sets, we always have $P_1$ close to $P_2$, then we can think Theorem \ref{thm:final} gives a good estimation of the reasonable value of $L_1$ when arm set is known. To confirm it, we did some experiments for different arm sets and the result is recorded in Table \ref{tb:L_1}. The implement details of this experiment can be found in Appendix \ref{app:exp}.

\begin{table}[H]
    \centering
    \begin{tabular}{|c|c|c|c|c|c|}
        \hline
        $K$ & 100 & 200 & 300 & 400 & 500 \\
        \hline
        $P_1$ & 7 & 14 & 19 & 24 & 26 \\
        \hline
        $P_2$ & 7 & 15 & 21 & 20 & 25 \\
        \hline
        difference & 0 & 1 & 2 & 4 & 1 \\
        \hline
    \end{tabular}
    \caption{Comparison of $P_1$ and $P_2$} \label{tb:L_1}
\end{table}

\begin{figure}[t!]
    \centering
    \includegraphics[width=0.45\textwidth]{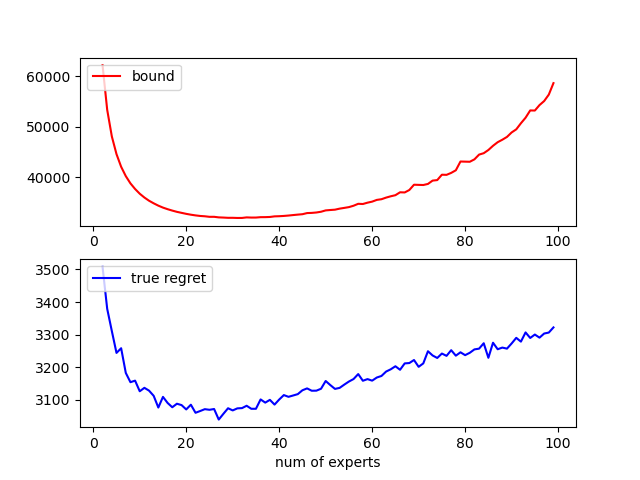}
    \caption{The curve of $M_i$ and $N_i$}
    \label{fig:P_12}
\end{figure}

Note that the best value of $L_1$ is not only related to $K$, but also effected by the expected means of arms. So it is not strange if larger $K$ corresponds to smaller $P_1$ or $P_2$.

Figure \ref{fig:P_12} is the curve of $M_i$ and $N_i$ in one process with $K = 500$ (not one process in Table \ref{tb:L_1}). We can see that the image of $M_i$ and $N_i$ have the same trend, which means the bound in Theorem \ref{thm:final} is reasonable. From Table \ref{tb:L_1} we find that $P_1$ is always close to $P_2$. Therefore, $P_1$ does give a good estimation of the best value of $L_1$ when arm set is known. It means Theorem \ref{thm:final} is an good approximate solution for the question.

\section{Conclusion} \label{sc:con}
In this paper, we consider an extension of standard bandit problem, which is called hierarchical experts bandit problem. We firstly show that we need to make strong assumptions about experts if we want the total regret not to grow linearly with the number of layers. Next, we analyze the hierarchical UCB structure and give two sub-linear regret upper bounds for different cases. Finally, we use some experiments to help analyze the regret of general cases of hierarchical UCB structure, and show that our theoretical results have some guiding significance for reasonable hierarchical decision structure. This can be helpful for applications such as company recruitment.

In our work, the main challenge is that the rewards of experts do not obey a fixed distribution, so we cannot directly utilize the various conclusions about the algorithms in stochastic case. To overcome this challenge, we assume that there is only one good expert in each layer, and then analyze the lower bound of the reward of this good expert and the upper bound of the rewards of other bad experts. In this way, the experts will have some properties similar to arms, and then we can use the various properties of the algorithms.

However, our bounds are still not tight, so there is some room for improvement. Our next step is to improve the upper bounds through more detailed proofs, and try to analyze the regret bounds for general hierarchical UCB structure according to the ideas provided by the experiments, e.g., clarifying the selection ranges of experts in hierarchical UCB structure or identifying the cases where $R_2 < R_1$ (defined in Experiment \ref{exp1}).


\appendix

\section{Complete proofs for proposed theorems} \label{app:proof}

\subsection{Proof for Theorem \ref{thm:bad}}
\begin{proof}
We first consider the case that $L_k = 3$, $\forall 1 \leq k \leq R$. Layer $k$ is $\{a_1^k, a_2^k, a_3^k\}$, where $a_1^k$ is reasonable and stable. $a_2^k$ and $a_3^k$ are ``bad'' experts, and their strategies follow: When $k < R$, $a_2^k$ and $a_3^k$ always selects $a_3^{k+1}$. $a_2^R$ and $a_3^R$ always selects arm $K$.

We use $T_i(m)$ to denote the number of times arm $i$ is selected during the first $m$ rounds, $\forall 1 \leq i \leq K$. For the selection range (arms/experts) $\{a_1,\cdots,a_k\}$ of $A$, we denote the number of times it selects $a_i$ during the first $m$ rounds by $T_i^A(m)$. Besides, we use $T_j^{k-1}(m)$ to denote the number of times $a_j^k$ is selected during the first $m$ rounds, $\forall 1 \leq k \leq R$. Let $D_{\mu_1,\mu_K} = \frac{1}{\mathrm{kl}\left(\mu_{K}, \mu_{1}\right)}$.

Note that $\forall t > 0$, $\forall 1 \leq k \leq R$, $\mathbb{E}[a^k_{2,t}] = \mathbb{E}[a^k_{3,t}] = \mu_K$. So from the reasonability of $a_1^k$ and $B$, we have $\mathbb{E}\left[T_1^B(n)\right] > \frac{1}{4}n, \mathbb{E}\left[T_1^1(n)\right] > (\frac{1}{4})^2 n, \cdots, \mathbb{E}\left[T_1^{R-1}(n)\right] > (\frac{1}{4})^R n$ when $n$ is large enough.

Then from the stability of $a_1^k$ and $B$ we have $\mathbb{E}\left[T_2^B(n)\right] \geq \frac{D_{\mu_1,\mu_K}}{2} \ln n, \mathbb{E}\left[T_2^1(n)\right] \geq \frac{D_{\mu_1,\mu_K}}{2} \left(\ln n - \ln 4\right), \cdots, \mathbb{E}\left[T_2^{R-1}(n)\right] \geq \frac{D_{\mu_1,\mu_K}}{2} \left(\ln n - (R-1) \ln 4\right)$ when $n$ is large enough.

~

So the total regret
\begin{align}
    R_n^{*} &\geq (\mu_1 - \mu_K) \sum_{k=0}^{R-1} \mathbb{E}\left[T_2^{k}(n)\right]  \nonumber\\
    &\geq \frac{\mu_1 - \mu_K}{2 \mathrm{kl}(\mu_K, \mu_1)} (R \ln n - R^2 \ln 4) \nonumber
\end{align}

This finishes the proof of the case that $L_k = 3$, $\forall 1 \leq k \leq R$.

For the case that some $L_k \geq 4$, we can regard $a_1^k, \cdots, a_{L_k-2}^k$ (the last $2$ experts are ``bad'') as one expert $a^k$. $a^k$ is selected if some $a_l^k$ is selected and a selection is from $a^k$ if it is from some $a_l^k$, $1 \leq l \leq L_k-2$. 

For the selection range $\{b_1,\cdots,b_L\}$ of $a^k$, suppose $\exists 1 \leq i \leq L$ s.t. $\mathbb{E}\left[b_{i,t}\right] = \underset{1 \leq j \leq L}{\min} \mathbb{E}[b_{j,t}], \forall t > 0$. When $a^k$ is chosen for $n > 1$ times, suppose $a_j^k$ is chosen for $n_j$ times, then $\exists C_j$ s.t. the expectation of the number of times $a_j^k$ chooses $b_i$ is no more than $\frac{n_j}{L} + C_j$. So the expectation of the number of times $a^k$ chooses $b_i$ is no more than $\frac{n}{L} + \sum_{j=1}^{L_k-2} C_j$. So $a^k$ is reasonable.

For the selection range $\{b_1,\cdots,b_L\}$ of $a^k$, suppose $\exists \eta_1 > \eta_2$, $1 \leq i \leq L$ s.t. $\underset{1 \leq j \leq L}{\max} \mathbb{E}[b_{j,t}] \leq \eta_1, \forall t > 0$ and $\mathbb{E}[b_{i,t}] \geq \eta_2, \forall t > 0$. When $a^k$ is chosen for $n > 1$ times, $\exists j$ s.t. $a_j^k$ is chosen for no less than $\frac{n}{L_k}$ times. Then $\exists C_j$ s.t. the number of times $b_i$ is chosen is no less than $\frac{1}{\mathrm{kl}\left(\eta_{2}, \eta_{1}\right)} \ln n - \frac{1}{\mathrm{kl}\left(\eta_{2}, \eta_{1}\right)} \ln L_k - C_j$. So $a^k$ is stable.

So $a^k$ is an reasonable and stable strategy and we can get the conclusion by the argument for $L_k = 3, \forall 1 \leq k \leq R$ case.
\end{proof}

\subsection{Proof for Theorem \ref{thm:ucb_jg}}
\begin{proof}
    Recall that $I_t$ denotes the arm selected at time $t$, $T_i(m)$ denotes the number of times arm $i$ is selected during the first $m$ rounds, and $S_i(m)$ denotes the total reward of arm $i$ during the first $m$ rounds. Suppose $J_t = 1$ and $I_t = i \geq 2$, then at least one of the three following inequalitys must be true:
    \begin{numcases}{}
        \frac{S_1(t-1)}{T_1(t-1)} + \sqrt{\frac{\alpha \ln t}{2 T_1(t-1)}} \leq \mu_1 \label{case1}\\
        \frac{S_i(t-1)}{T_i(t-1)} - \sqrt{\frac{\alpha \ln t}{2 T_i(t-1)}} > \mu_i \label{case2}\\
        T_{i}(t-1)<\frac{2 \alpha \ln t}{\Delta_{i}^{2}} \label{case3}
    \end{numcases}
    This is because when the three inequalitys are all false, we have
    \begin{align}
        \frac{S_1(t-1)}{T_1(t-1)} + \sqrt{\frac{\alpha \ln t}{2 T_1(t-1)}} &> \mu_1 \nonumber \left(=\mu_{i}+\Delta_{i}\right) \nonumber\\
        & \geq \mu_{i}+2 \sqrt{\frac{\alpha \ln t}{2 T_i(t-1)}} \nonumber\\
        & \geq \frac{S_i(t-1)}{T_i(t-1)} + \sqrt{\frac{\alpha \ln t}{2 T_i(t-1)}} \nonumber
    \end{align}
    which contradicts with UCB's choice. Hence,
    \begin{align}
        \mathbb{E} \left[T_{i}(n)\right] =& \mathbb{E} \left[\sum_{t=1}^{n} \mathbb{I}_{I_{t}=i}\right] \nonumber\\
        \leq& \sum_{t=1}^{n} \left(\mathbb{P}(\text{(\ref{case1}) is true}) +\mathbb{P}(\text{(\ref{case2}) is true})\right) + \nonumber\\
        &\mathbb{E} \left[\sum_{t=1}^{n} \mathbb{I}_{I_{t}=i, J_t = 1, \text{(\ref{case3}) is true}}\right] + \nonumber\\
        &\mathbb{E} \left[\sum_{t=1}^{n} \mathbb{I}_{I_{t}=i, J_t \neq 1}\right] \nonumber
    \end{align}
    Denote
    $$A_n = \sum_{t=1}^{n} \mathbb{I}_{I_{t}=i, J_t = 1, \text{(\ref{case3}) is true}} \quad B_n = \sum_{t=1}^{n} \mathbb{I}_{I_{t}=i, J_t \neq 1}$$
    Let Condition $i$ (about $m$) be: $C(t,m) \leq \frac{2 \alpha \ln m}{\Delta_i^2} - \frac{2 \alpha \ln t}{\Delta_i^2}$, $\forall 1 \leq t \leq m$. Then according to the conditions in this theorem, $n$ satisfies Condition $i$, $\forall 2 \leq i \leq K$. \\
    Next, we want to prove $A_n + B_n \leq l+1, l = \lceil \frac{2 \alpha \ln n}{\Delta_i^2} \rceil$ by induction on $l$. \\
    When $l=0$, we must have $n=1$. So $A_n+B_n \leq n \leq 1$. \\
    Suppose $\forall l \leq k-1, A_n+B_n \leq l+1$. \\
    For $l=k$, let $n_k = \underset{t=1, \ldots, n}{\operatorname{argmax}}\left\{\frac{2 \alpha \ln t}{\Delta_i^2} \leq k-1\right\}$, $t_n = \underset{t=1, \ldots, n}{\operatorname{argmax}}\left\{J_t \neq 1\right\}$. \\
    If $t_n \leq n_k$, changing $J_{t_n}$ to $1$ makes $n_k$ satisfies Condition $i$: After changing $J_{t_n}$ to $1$, $C(t,n_k) \leq C(t,n) \leq \operatorname{max}\left\{0,\frac{2 \alpha \ln n}{\Delta_i^2} - \frac{2 \alpha \ln t}{\Delta_i^2} -1\right\} \leq \frac{2 \alpha \ln n_k}{\Delta_i^2} - \frac{2 \alpha \ln t}{\Delta_i^2}$. By inductive hypothesis, $A_{n_k} + B_{n_k} = x \leq k+1$. From definition, $A_n - A_{n_k} \leq max\{0,k-x\}, B_n - B_{n_k} = 0$. So $A_n + B_n \leq k+1$. \\
    If $t_n > n_k$, then $C(t,n_k) \leq \operatorname{max}\{0,C(t,n)-1\}$ and $n_k$ satisfies Condition $i$. $A_{n_k} + B_{n_k} = x \leq k, A_n - A_{n_k} \leq k-x, B_n - B_{n_k} \leq C(n_k,n) \leq 1$. So $A_n + B_n \leq k+1$. \\
    This finishes the induction and we derive that $A_n + B_n \leq \frac{2 \alpha \ln n}{\Delta_i^2}+2$. \\
    Finally, we prove $\sum_{t=1}^{n} \left(\mathbb{P}(\text{(\ref{case1}) is true}) +\mathbb{P}(\text{(\ref{case2}) is true})\right)$ can be bounded by constant using Hoeffding's inequality: $\forall t$,
    \begin{align}
    \mathbb{P}(\text{(\ref{case1}) is true}) &\leq \sum_{s=1}^{t} \mathbb{P}\left(\frac{\sum_{k=1}^{s} X_{1,k}}{s}+\sqrt{\frac{\alpha \ln t}{2 s}} \leq \mu_1 \right) \nonumber\\
    &\leq \sum_{s=1}^{t} t^{-\alpha} \leq t^{-\alpha+1} \nonumber
    \end{align}
    So $\sum_{t=1}^{n} \mathbb{P}(\text{(\ref{case1}) is true}) \leq \sum_{t=1}^{n} t^{-\alpha+1} \leq C_{\alpha}$ when $\alpha > 2$. \\ Similarly, $\sum_{t=1}^{n} \mathbb{P}(\text{(\ref{case2}) is true}) \leq C_{\alpha}$. \\
    Overall, we have
    $$\mathbb{E} \left[T_{i}(n)\right] \leq \frac{2 \alpha \ln n}{\Delta_i^2} + 2 C_{\alpha} + 2$$
\end{proof}

\subsection{Proof for Theorem \ref{thm:final}}
We firstly assume arms are deterministic (return the same value each time) and prove a stronger version of Theorem \ref{thm:final}:

\begin{theorem} \label{thm:122}
    Suppose there are $R$ layers of experts with $1<L_1<L_2<\cdots<L_R<K$. $a_j^k$ is $\alpha_j^k$-UCB (Here the $k$ in $\alpha_j^k$ is an index, not the $k$ power of $\alpha_j$), $\forall 1 \leq j \leq L_k, 1 \leq k \leq R$. Top strategy is $\beta$-UCB. Arms are deterministic. \\
    Then the following conclusion holds: $\forall \varepsilon > 0$, $\exists$ constant $M_{\varepsilon}$, if $\alpha_j^k > M_{\varepsilon}, \forall 2 \leq j \leq L_k, 1 \leq k \leq R$, then $\exists$ constant $C_{\varepsilon}$ s.t. $\forall n > 0$ we have \\
    If $i^{*} > 2$,
    $$R_n^* \leq \left(\frac{\alpha_1^R}{2 \Delta_{i^{*}-1}^2} \left(\sum_{l=i^{*}}^{K} \Delta_l\right) + \sum_{i=2}^{i^{*}-1} \frac{\alpha_1^R}{2 \Delta_i} + \varepsilon \right) \ln n + C_{\varepsilon}$$
    Else,
    $$R_{n}^{*} \leq\left(\frac{\sum_{k \in \mathcal{S}_{2}}(L_k-1)\alpha_1^{k-1}}{2\Delta_{2}^2/\Delta_{K}} +\sum_{i=2}^{K} \frac{\alpha_{1}^R}{2 \Delta_{i}}+\varepsilon\right) \ln n + C_{\varepsilon}$$
    where
    $$i^* = \min \left\{2 \leq i \leq K: (K-i)\frac{\alpha_1^R}{\Delta_{i}^{2}} - \sum_{l=i+1}^{K}\frac{\alpha_1^R}{\Delta_{l}^{2}} \leq \right.$$
    $$\left.\frac{\sum_{k \in \mathcal{S}_i}(L_k-1)\alpha_1^{k-1}}{\Delta_i^2} \right\}, \quad \alpha_1^{0} = \beta$$
    $$\mathcal{S}_m = \left\{1 \leq k \leq R: (L_k-1)\frac{\alpha_1^{k-1}}{\Delta_m^2} \geq \max\{\right.$$
    $$\left.(L_1-1)\frac{\beta}{\Delta_K^2}, (L_2-1)\frac{\alpha_1^1}{\Delta_K^2}, \cdots, (L_{k-1}-1)\frac{\alpha_1^{k-2}}{\Delta_K^2} \} \right\}$$
\end{theorem}

~

To prove this theorem, we firstly propose a lemma.

\begin{lemma} \label{lm:ucb_ul}
    If an $\alpha$-UCB controls an deterministic arm set, then (1) $\forall i \geq 2$, $T_i(n) \leq \frac{\alpha}{2 \Delta_{i}^{2}} \ln n+1$. (2) $\forall \varepsilon > 0$, $T_{i}(n) \geq \frac{\alpha}{2\left(\Delta_{i}+\varepsilon\right)^{2}} \ln n$ when $n$ is large enough.
\end{lemma}

\begin{proof}
$\forall i \geq 1$, Let $U_i(n)=\mu_i+\sqrt{\frac{\alpha \ln n}{2 T_i(n-1)}}$. \\
Suppose $\alpha$-UCB selects arm $i$ for the $k$th time at round $t_k^i$. Then
$$\mu_{i}+\sqrt{\frac{\alpha \ln t_{k+1}^{i}}{2 k}} \geq \mu_1+\sqrt{\frac{\alpha \ln t_{k+1}^{i}}{2 T_1(t_{k+1}^{i}-1)}}>\mu_{1}$$
$$\Rightarrow \ln t_{k+1}^{i} > \frac{2 \Delta_{i}^{2}}{\alpha} k$$
So $T_{i}(n) \leq k$ when $\ln n \leq \frac{2 \Delta_{i}^{2}}{\alpha} k$. \\
So $$T_i(n) \leq \lceil \frac{\alpha}{2 \Delta_{i}^{2}} \ln n \rceil \leq \frac{\alpha}{2 \Delta_{i}^{2}} \ln n+1$$
So $$\frac{\ln n}{ T_{1}(n)} \rightarrow 0, U_{1}(n) \rightarrow \mu_{1}, n \rightarrow \infty$$

We then prove that when $k$ is large enough, $t_{k+1}^1 - t_k^1 \leq 3K$. \\
Suppose arm $1$ is pulled at round $n$ and has not been pulled between round $n+1$ and $n+3K$. Then $\exists i \geq 2$ s.t. arm $i$ is successively pulled at round $n+k_1$, $n+k_2$, $n+k_3$ with $T_i(n+k_1-1) = T_i(n)$ and $k_3 \leq 3K$. Then
\begin{numcases}{}
    \frac{\Delta_{i}}{\sqrt{\alpha/2}} \geq \sqrt{\frac{\ln n}{T_{i}(n-1)}}-\sqrt{\frac{\ln n}{T_1(n-1)}} \nonumber\\
    \frac{\Delta_{i}}{\sqrt{\alpha/2}} \leq \sqrt{\frac{\ln (n+k_3)}{T_{i}(n-1)+2}}-\sqrt{\frac{\ln (n+k_3)}{T_1(n-1)+1}} \nonumber
\end{numcases}
Because
$$\sqrt{\frac{\ln (n+k_3)}{T_{i}(n-1)+1}}-\sqrt{\frac{\ln (n+k_3)}{T_1(n-1)}} \geq $$
$$\sqrt{\frac{\ln (n+k_3)}{T_{i}(n-1)+2}}-\sqrt{\frac{\ln (n+k_3)}{T_1(n-1)+1}}$$
we have
$$\sqrt{\frac{\ln (n+k_3)}{T_{i}(n-1)+1}}-\sqrt{\frac{\ln (n+k_3)}{T_1(n-1)}} \geq $$
$$\sqrt{\frac{\ln n}{T_{i}(n-1)}}-\sqrt{\frac{\ln n}{T_1(n-1)}}$$
$$\Rightarrow \sqrt{\frac{\ln (n+k_3)}{T_{i}(n-1)+1}} \geq \sqrt{\frac{\ln n}{T_{i}(n-1)}}$$
$$\Rightarrow \frac{\ln (n+k_3)}{\ln n} \geq \frac{T_{i}(n-1)+1}{T_{i}(n-1)}$$
\begin{align}
\Rightarrow n+k_3 \geq n^{1+\frac{1}{T_{i}(n-1)}} &\geq n + \frac{n}{T_{i}(n-1)} \nonumber\\
&\geq n + \frac{n}{\frac{\alpha}{2 \Delta_{i}^{2}} \ln n+1} \nonumber
\end{align}
When $n$ is large enough, we have $k_3 > 3K$. It is a contradiction. \\

$\forall \varepsilon > 0$, $\exists N$ s.t. $U_{1}(n)<\mu_1+\varepsilon, \forall n > N$.
When $\alpha$-UCB pulls arm $1$ at round $n > N$, 
$$\forall i \geq 2, \quad \mu_1 + \varepsilon > U_1(n) \geq U_i(n) = \mu_i + \sqrt{\frac{\alpha \ln n}{2 T_i(n-1)}}$$
$$\Rightarrow T_i(n) \geq \frac{\alpha}{2\left(\Delta_{i}+\varepsilon\right)^{2}} \ln n$$
When $N$ is large enough, $\forall n > N+3K$, $\exists k \leq 3K$ s.t. arm $1$ is pulled at round $n-k$, and
\begin{align}
T_i(n) \geq T_i(n-k) &\geq \frac{\alpha}{2\left(\Delta_{i}+\varepsilon\right)^{2}} \ln (n-k) \nonumber\\
&\geq \frac{\alpha}{2\left(\Delta_{i}+2\varepsilon\right)^{2}} \ln n \nonumber
\end{align}
\end{proof}



Then we can prove Theorem \ref{thm:122}.

\begin{proof}
    Denote the top stragy as $B$. Recall that $T_i(m)$ denote the number of times arm $i$ is selected during the first $m$ rounds, and $T_j^{k-1}(m)$ denotes the number of times $a_j^k$ is selected during the first $m$ rounds. We use $\hat{\mu}_{a_j^k}(m)$ to denote the average reward of $a_j^k$ during the first $m$ rounds. \\
    For fixed $n$, when $\alpha_j^k$ is large enough, $a_j^k$ will always select the expert/arm with the least number of pulls at time $m < n$. We firstly assume that $\forall 1 \leq k \leq R, \forall j \geq 2, a_j^k$ always selects the arm with the least number of pulls. \\
    For simplicity, we use $a_j$ to denote $a_j^R$ and $\alpha_j$ to denote $\alpha_j^R$ in the following. \\
    
    We first prove that $\forall \varepsilon > 0, \hat{\mu}_{a_1}(n-1) > \mu_1 - \varepsilon$ when $n$ is large enough. \\
    $\forall \varepsilon > 0, \exists q > 0$ s.t. $\mu_1 - \varepsilon = \frac{q \mu_{1}+\mu_{K}}{q+1}$. $\forall N \in \mathbb{Z}^{+}$, suppose $\hat{\mu}_{a_1}(N-1) \leq \frac{q \mu_{1}+\mu_{K}}{q+1}$. Suppose $a_1$ selects arm $i$ for the last time before round $N$ at round $M_i < N$ (If $M_i$ does not exist, let $M_i = 0$). Let $M = \max \{M_2, \cdots, M_K\}$. \\
    For $i \geq 2$ with $M_i > 0$,
    $$\mu_{i}+\sqrt{\frac{\alpha_{1} \ln M_i}{2 T_{i}(M_{i}-1)}} \geq \mu_1 \Rightarrow T_{i}(M_{i}-1) \leq \frac{\alpha_1}{2 \Delta_i^2} \ln M_i$$
    So $$\sum_{i=2}^{K} T_i^{a_1}(M-1) \leq \frac{1}{q+1} C_{\alpha_1} \ln M, C_{\alpha_1} = (q+1) \sum_{i=2}^{K} \frac{\alpha_1}{2 \Delta_i^2}$$
    $$\Rightarrow T_1^{R-1}(M-1) \leq C_{\alpha_1} \ln M,$$
    $$\sum_{j=2}^{L_R} T_j^{R-1}(M-1) \geq M-1- C_{\alpha_1} \ln M$$
    Suppose $M = M_{j_1}$, if $\frac{M-1}{2K} > C_{\alpha_1} \ln M$, we consider the arm with the largest number of pulls before round $M$ and call it arm $i_1$. Then $T_{i_1}(M-1) \geq \frac{M-1}{K}$. When arm $i_1$ is pulled by some $a_h$ with $h \geq 2$ for the last time before round $M$ at round $m_1$, $T_{i_1}(m_1-1) \geq \frac{M-1}{K} - C_{\alpha_1} \ln M > C_{\alpha_1} \ln M$. \\
    Because $a_h$ always selects the arm with fewest pulls, $T_{j_1}(M-1) \geq T_{j_1}(m_1-1) > C_{\alpha_1} \ln M$. But from the definition of $j_1$ we have $T_{j_1}(M-1) \leq \frac{\alpha_{1}}{2 \Delta_{j_1}^{2}} \ln M < C_{\alpha_1} \ln M$, which is a contradiction. \\
    So $\frac{M-1}{2K} \leq C_{\alpha_1} \ln M$, which means $M$ is no more than a constant $C_1$. \\
    Hence $T_1^{R-1}(N-1) \leq (q+1) C_1$. However, from Lemma \ref{lm:ucb_ul} we know $T_1^{R-1}(N-1) \geq \frac{\alpha_1^{R-1}}{3 \Delta_K^2} \ln(N-1)$ when $N$ is large enough. Let $N \rightarrow \infty$ and we get a contradiction. \\
    So when $n$ is large enough, we have $\hat{\mu}_{a_1}(n-1) > \mu_1 - \varepsilon$. \\
    
    From above argument we know that $\forall \varepsilon > 0, \exists N \in \mathbb{Z}^{+}$ s.t. $\forall m \geq N, \hat{\mu}_{a_1}(m-1) > \mu_1 - \varepsilon$. \\
    We next prove $\forall j \geq 2, T_{j}^{R-1}(n-1) \leq T_{1}^{R-1}(n-1) + 1$ when $n$ is large enough. \\
    Let $h=\underset{1 \leq j \leq L_{R}}{\operatorname{argmax}}$ $T_{j}^{R-1}(N-1)$. Suppose $h \geq 2$ and $T_{h}^{R-1}(n-1) > T_{1}^{R-1}(n-1) + 1$. \\
    If $\hat{\mu}_{a_h}(n-1) \leq \mu_1 - 2 \varepsilon$, \\
    $\exists D$ s.t. $\forall 1 \leq l \leq L_R, \forall m > 0$, if $T_{l}^{R-1}(m-1) > D$ and $a_l$ is pulled at round $m$, then $\left| \hat{\mu}_{a_{l}}(m-1) - \hat{\mu}_{a_{l}}(m) \right| < \frac{\varepsilon}{2}$. \\
    We assume $n$ s.t. $T_{h}^{R-1}(n-1) > \max\{D,N\} + 1$. Consider the last pull of $a_{h}$ before round $n$ at round $n^{\prime}$. Then $T_{h}^{R-1}\left(n^{\prime}-1\right) \geq T_{1}^{R-1}\left(n^{\prime}-1\right)$ and $\hat{\mu}_{a_{h}}(n'-1)<\hat{\mu}_{a_{1}}\left(n^{\prime}-1\right)-\frac{\varepsilon}{2}$. It contradicts the pull of $a_h$. \\
    So $\hat{\mu}_{a_h}(n-1) > \mu_1 - 2 \varepsilon$ when $n$ is large enough. Note that $T_{h}^{R-1}(n-1) \geq \frac{n-1}{L_R}$. \\
    Because $\hat{\mu}_{a_h}(n-1) > \mu_1 - 2 \varepsilon$, $T_1^{a_h}(n-1) > \left(1-\frac{2\varepsilon}{\Delta_2}\right)T_{h}^{R-1}(n-1) \geq \left(1-\frac{2\varepsilon}{\Delta_2}\right) \frac{n-1}{L_R}$. \\
    Note that $L_R < K$, when $\varepsilon << 1$, $T_1^{a_h}(n-1) > \frac{n-1}{K}+1$. From the property of $a_h$ (always selecting the expert/arm with the least number of pulls) we have $\sum_{i=1}^{K} T_i(n-1) \geq K \left(T_1^{a_h}(n-1)-1\right) > n-1$, which is a contradiction. \\
    So $\forall j \geq 2, T_{j}^{R-1}(n-1) \leq T_{1}^{R-1}(n-1) + 1$ when $n >> 1$, $T_{1}^{R-1}(n-1) \geq \frac{n-1}{L} -1$. \\
    Because $\forall i \geq 2, T_i^{a_1}(n-1) \leq \frac{\alpha_1}{2 \Delta_{i}^{2}} \ln n+1$, $T_1(n-1) > \frac{n-1}{K-1/2}$ when $n >> 1$. \\
    Let $b = \frac{K}{K-1/3} > 1$, from the property of $a_h$ with $h \geq 2$, $a_j$ will not select arm $1$ between round $n$ and $\lceil b n \rceil$ if $j \geq 2$. \\
    From the same inference we get $a_j$ will not select arm $1$ between round $\lceil b n \rceil$ and $\lceil b^2 n \rceil$ if $j \geq 2$. Repeat the process and we get $a_j$ will never select arm $1$ after round $n$ if $j \geq 2$. \\
    Hence, $\forall \varepsilon > 0, \exists N \in \mathbb{Z}^{+}$ s.t. $\forall m \geq N, \hat{\mu}_{a_1}(m-1) > \mu_1 - \varepsilon, \hat{\mu}_{a_j}(m-1) < \mu_2 + \varepsilon, \forall j \geq 2$. \\
    
    Then from similar inference we get $\forall \varepsilon > 0, \exists N \in \mathbb{Z}^{+}$ s.t. $\forall m \geq N, \hat{\mu}_{a_1^{R-1}}(m-1) > \mu_1 - 2 \varepsilon, \hat{\mu}_{a_j^{R-1}}(m-1) < \mu_2 + 2 \varepsilon, \forall j \geq 2$. \\
    Repeat the process and we finally get $\forall \varepsilon > 0, \exists N \in \mathbb{Z}^{+}$ s.t. $\forall 1 \leq k \leq R, \forall m \geq N, \hat{\mu}_{a_1^{k}}(m-1) > \mu_1 - (R-k+1) \varepsilon, \hat{\mu}_{a_j^{k}}(m-1) < \mu_2 + (R-k+1) \varepsilon, \forall j \geq 2$. \\
    
    From above argument and Lemma \ref{lm:ucb_ul}, we have $\forall \varepsilon > 0, \exists N \in \mathbb{Z}^{+}$ s.t. $\forall 1 \leq k \leq R, \forall m \geq N$, $\hat{\mu}_{a_1^{k}}(m-1) > \mu_1 - (R-k+1) \varepsilon$, $\hat{\mu}_{a_j^{k}}(m-1) < \mu_2 + (R-k+1) \varepsilon$, $T_j^B(m-1) < \frac{\beta}{2\left(\Delta_2 - (2R+1)\varepsilon\right)^2} \ln m, \forall j \geq 2$. \\
    So $\sum_{j=2}^{L_1} T_j^B(m-1) < \frac{(L_1-1)\beta}{2\left(\Delta_2 - (2R+1)\varepsilon\right)^2} \ln m$. $\forall i \geq 2, T_i^{a_1^1}(m-1) < \frac{\alpha_1^1}{2\left(\Delta_2 - (2R-1)\varepsilon\right)^2} \ln m$. \\
    From Lemma \ref{lm:ucb_ul}, $\sum_{j=2}^{L_1} T_j^B(m-1) \geq \frac{(L_1-1)\beta}{2\left(\Delta_K + \varepsilon\right)^2} \ln m$ when $N$ is large enough. \\
    If $(L_2-1)\frac{\alpha_1^1}{2 \Delta_2^2} < (L_1-1)\frac{\beta}{2 \Delta_K^2}$, from the property of $a_h^1$ with $h \geq 2$ we have $\forall i \geq 2, T_i^2(m-1) \geq \frac{(L_1-1)\beta}{2(L_2-1)\left(\Delta_K + \varepsilon\right)^2} \ln m - 1 > \frac{\alpha_1^1}{2\left(\Delta_2 - (2R-1)\varepsilon\right)^2} \ln m$ when $\varepsilon << 1$. \\
    Then $a_1^1$ will not select $a_i^2$ at round $m$, $\forall i \geq 2$. So $\exists$ constant $C_{\varepsilon}$ s.t.
    $$\sum_{i=2}^{L_2} T_i^1(m-1) < \frac{(L_1-1)\beta}{2\left(\Delta_2 - (2R+1)\varepsilon\right)^2} \ln m + C_{\varepsilon}$$
    If $(L_2-1)\frac{\alpha_1^1}{2 \Delta_2^2} \geq (L_1-1)\frac{\beta}{2 \Delta_K^2}$,
    $$\sum_{i=2}^{L_2} T_i^1(m-1) < \frac{(L_1-1)\beta+(L_2-1)\alpha_1^1}{2\left(\Delta_2 - (2R+1)\varepsilon\right)^2} \ln m + C_{\varepsilon}$$
    From similar inference we get: \\
    If $(L_3-1)\frac{\alpha_1^2}{2 \Delta_2^2} < \max\left\{(L_2-1)\frac{\alpha_1^1}{2 \Delta_K^2}, (L_1-1)\frac{\beta}{2 \Delta_K^2}\right\}$, $\exists$ constant $C_{\varepsilon}^{'}$ s.t.
    $$\sum_{i=2}^{L_3} T_i^2(m-1) < \sum_{i=2}^{L_2} T_i^1(m-1) + C_{\varepsilon}^{'}$$
    Else,
    \begin{align}
    \sum_{i=2}^{L_3} T_i^2(m-1) < &\sum_{i=2}^{L_2} T_i^1(m-1) + \nonumber\\
    &\frac{(L_3-1)\alpha_1^2}{2\left(\Delta_2 - (2R+1)\varepsilon\right)^2} \ln m + C_{\varepsilon}^{'} \nonumber
    \end{align}
    Repeat the process and we finally get ($m >> 1$):
    $$\sum_{i=2}^{L_R} T_i^{R-1}(m-1) < \frac{\sum_{k \in \mathcal{S}_2}(L_k-1)\alpha_1^{k-1}}{2\left(\Delta_2 - (2R+2)\varepsilon\right)^2} \ln m, \alpha_1^{0} = \beta$$
    where
    $$\mathcal{S}_m = \left\{1 \leq k \leq R: (L_k-1)\frac{\alpha_1^{k-1}}{\Delta_m^2} \geq \max\{\right.$$
    $$\left.(L_1-1)\frac{\beta}{\Delta_K^2}, (L_2-1)\frac{\alpha_1^1}{\Delta_K^2}, \cdots, (L_{k-1}-1)\frac{\alpha_1^{k-2}}{\Delta_K^2} \} \right\}$$
    
    Suppose $a_h^R$ $(h \geq 2)$ selects arm $i$ at round $m$, then $\forall l > i, T_{l}(m-1) \geq T_i(m-1)$. \\
    From Lemma \ref{lm:ucb_ul}, when $m >> 1$, $T_i(m-1) \geq \left(\frac{\alpha_1^R}{2\Delta_{i}^{2}} - \varepsilon\right) \ln m$ and $T^{a_1^R}_{l}(m-1) \leq \frac{\alpha_1^R}{2 \Delta_{l}^{2}} \ln m+1, \forall l > i$. So
    $\sum_{l=i+1}^{K} \left(\left(\frac{\alpha_1^R}{2\Delta_{i}^{2}} - \varepsilon \right) \ln m - \frac{\alpha_1^R}{2 \Delta_{l}^{2}} \ln m -1\right) \leq \sum_{j=2}^{L_R} T_j^{R-1}(m-1)$
    $$\Rightarrow (K-i)\frac{\alpha_1^R}{\Delta_{i}^{2}} - \sum_{l=i+1}^{K}\frac{\alpha_1^R}{\Delta_{l}^{2}} \leq \frac{\sum_{k \in \mathcal{S}_2}(L_k-1)\alpha_1^{k-1}}{\Delta_2^2} + O(\varepsilon)$$
    Let $$i_m = \min \left\{2 \leq i \leq K: (K-i)\frac{\alpha_1^R}{\Delta_{i}^{2}} - \sum_{l=i+1}^{K}\frac{\alpha_1^R}{\Delta_{l}^{2}} \leq \right.$$
    $$\left.\frac{\sum_{k \in \mathcal{S}_m}(L_k-1)\alpha_1^{k-1}}{\Delta_m^2} \right\}$$
    Then $i \geq i_2$. \\
    If $i_2 > 2$, we have $\forall j \geq 2, a_j^R$ will never select arm $1,2$ after round $N$. \\
    Then $\forall \varepsilon > 0, \exists N_2 \in \mathbb{Z}^{+}$ s.t. $\forall 1 \leq k \leq R, \forall m \geq N_2, \hat{\mu}_{a_1^{k}}(m-1) > \mu_1 - (R-k+1) \varepsilon, \hat{\mu}_{a_j^{k}}(m-1) < \mu_3 + (R-k+1) \varepsilon, T_j^B(m-1) < \frac{\beta}{2\left(\Delta_3 - (2R+1)\varepsilon\right)^2} \ln m, \forall j \geq 2$. From similar inference we can get that if $i_3 > 3$, $\forall j \geq 2, a_j^R$ will never select arm $1,2,3$ after round $N_2$. \\
    Repeat this process and we finally get $N^{*}$ s.t. $\forall j \geq 2, a_j^R$ will only select arm $i^{*}, \cdots, K$ after round $N^*$ with
    $$i^* = \min \left\{2 \leq i \leq K: (K-i)\frac{\alpha_1^R}{\Delta_{i}^{2}} - \sum_{l=i+1}^{K}\frac{\alpha_1^R}{\Delta_{l}^{2}} \leq \right.$$
    $$\left.\frac{\sum_{k \in \mathcal{S}_i}(L_k-1)\alpha_1^{k-1}}{\Delta_i^2} \right\}$$ \\
    
    Now we remove the assumption that $\forall 1 \leq k \leq R, \forall j \geq 2, a_j^k$ always selects the arm with the least number of pulls. \\
    From above argument we know $\forall \varepsilon > 0, \forall D > 0, \exists$ constant $M_{\varepsilon}$ s.t. when $\alpha_j^k > M_{\varepsilon}, \forall 2 \leq j \leq L_k, 1 \leq k \leq R$, $\exists N > D$ s.t. $\forall 1 \leq k \leq R, \hat{\mu}_{a_1^{k}}(N-1) > \mu_1 - (R-k+1) \varepsilon, \forall j \geq 2, \hat{\mu}_{a_j^{k}}(N-1) < \mu_{i^*} + (R-k+1) \varepsilon$ and $$\sum_{i=2}^{L_R} T_i^{R-1}(N-1) < \frac{\sum_{k \in \mathcal{S}_{i^*}}(L_k-1)\alpha_1^{k-1}}{2\left(\Delta_{i^*} - (2R+2) \varepsilon\right)^2} \ln N$$
    Suppose some $a_h^R$ with $h \geq 2$ is pulled for the first time after round $N-1$ at round $m_1 \geq N$. \\
    If $a_h^R$ selects arm $i < i^*$ at round $m_1$, then $\forall l \geq i^*$, when $N >> 1$,
    \begin{align}
    \mu_l + \sqrt{\frac{\alpha_h^R \ln m_1}{2 T_l(m_1-1)}} &\leq \mu_i + \sqrt{\frac{\alpha_h^R \ln m_1}{2 T_i(m_1-1)}} \nonumber\\
    &\leq \mu_i + (\Delta_i + \varepsilon) \sqrt{\frac{\alpha_h^R}{\alpha_1^R}} \nonumber
    \end{align}
    When $\alpha_h^R >> 1$, we have
    $$\frac{\ln m_1}{T_l(m_1-1)} \leq \frac{2(\Delta_i + 2\varepsilon)^2}{\alpha_1^R}$$
    $$\Rightarrow T_l(m_1-1) \geq \frac{\alpha_1^R}{2 (\Delta_i + 2\varepsilon)^2} \ln m_1$$
    When $\varepsilon << 1$, from the definition of $i^*$,
    $$\sum_{j=2}^{L_R} T_{j}^{R-1}(m_{1}-1) \geq$$
    $$\sum_{l=i^*}^{K} \left(\frac{\alpha_1^R}{2 (\Delta_i + 2\varepsilon)^2} \ln m_1 - \frac{\alpha_{1}^R}{2 \Delta_{l}^{2}} \ln m_1 -1 \right) >$$
    $$\frac{\sum_{k \in \mathcal{S}_{i^*}}(L_k-1)\alpha_1^{k-1}}{2\left(\Delta_{i^*} - (2R+2) \varepsilon\right)^2} \ln m_1$$
    It is a contradiction. \\
    So after round $N$, $\forall j \geq 2, a_j^R$ will never selects arm $i < i^*$. \\
    
    \noindent If $i^{*} > 2$, because $T_{K}(n-1) \leq T_{K-1}(n-1)+1 \leq \cdots \leq T_{i^{*}-1}(n-1)+K-i^{*}+1$,
    $$T_{l}(n-1) \leq \frac{\alpha_{1}^R}{2 \Delta_{i^{*}-1}^{2}} \ln n+K-i^{*}, \forall l \geq i^{*}$$
    Therefore, $\forall \varepsilon > 0$, $\exists$ constant $M_{\varepsilon}$, if $\alpha_j^k > M_{\varepsilon}, \forall 2 \leq j \leq L_k, 1 \leq k \leq R$, then $\exists$ constant $C_{\varepsilon}$ s.t. $\forall n > 0$ we have \\
    If $i^{*} > 2$,
    $$R_n^* \leq \left(\frac{\alpha_1^R}{2 \Delta_{i^{*}-1}^2} \left(\sum_{l=i^{*}}^{K} \Delta_l\right) + \sum_{i=2}^{i^{*}-1} \frac{\alpha_1^R}{2 \Delta_i} + \varepsilon \right) \ln n + C_{\varepsilon}$$
    Else,
    $$R_{n}^{*} \leq\left(\frac{\sum_{k \in \mathcal{S}_{2}}(L_k-1)\alpha_1^{k-1}}{2\Delta_{2}^2/\Delta_{K}} +\sum_{i=2}^{K} \frac{\alpha_{1}^R}{2 \Delta_{i}}+\varepsilon\right) \ln n + C_{\varepsilon}$$
\end{proof}

Finally, we use the result of Theorem \ref{thm:122} to prove Theorem \ref{thm:final}. The technique that using deterministic case to prove general case is learned from the proof of Theorem 1 in \cite{uct2}.

\begin{proof}
    Let $\delta_{n}=\frac{\delta}{2 K n(n+1)}, c_{n}=\sqrt{\frac{\ln \left(\delta_{n}^{-1}\right)}{2 n}}, \hat{\mu}_{i, n_{i}} = \frac{1}{n_i}\sum_{m=1}^{n_i} X_{i,m}$.
    Using Hoeffding's inequality (Fact \ref{fact:hoe}), $\forall 1 \leq i \leq K$,
    \begin{align}
    \mathbb{P}\left(\left|\sum_{m=1}^{n} X_{i, m}-n \mu_{i}\right| \geq n c_{n}\right) &\leq 2 \exp \left(-2 n c_{n}^{2}\right) \nonumber\\
    &=\frac{\delta}{K n(n+1)} \nonumber
    \end{align}
    Then
    \begin{align}
    \mathbb{P}\left(\exists n_{i} \geq 1,\left|\hat{\mu}_{i, n_{i}}-\mu_{i}\right| \geq c_{n_{i}}\right) &\leq \sum_{n_{i}=1}^{\infty} \frac{\delta}{K n_{i}\left(n_{i}+1\right)} \nonumber\\
    &=\frac{\delta}{K} \nonumber
    \end{align}
    $$\Rightarrow \sum_{i=1}^{K} \mathbb{P}\left(\exists n_{i} \geq 1,\left|\hat{\mu}_{i, n_{i}}-\mu_{i}\right| \geq c_{n_{i}}\right) \leq \delta$$
    So with probability $1-\delta$ we have $\forall m \geq 1, \forall 1 \leq i \leq K$, $\left|\hat{\mu}_{i, T_{i}(m)}-\mu_{i}\right|<c_{T_{i}(m)}$. \\
    $\forall \varepsilon>0, \exists N_{\varepsilon, \delta} \in \mathbb{Z}^{+}$s.t. $\forall n \geq N_{\varepsilon, \delta}$, $c_{n}<\varepsilon$. \\
    We state that $\forall D>0, \exists n>D$, The arm pulled at round $n$ has the most pulls before round $n$. \\
    In fact, $\exists n>D$ s.t. $\underset{1 \leq i \leq K}{\max} T_{i}(n-1) \neq \underset{1 \leq i \leq K}{\max} T_{i}(n)$, because $\underset{1 \leq i \leq K}{\max} T_{i}(K\lceil D\rceil+1) > \underset{1 \leq i \leq K}{\max} T_{i}(\lceil D\rceil)$. \\
    Then this $n$ satisfies the condition. \\
    Suppose arm $i$ is pulled by $a_j^R$ at round $n$, then $T_{i}(n-1) \geq \frac{n-1}{K}$. \\
    $\forall i' \neq i$, 
    $$\hat{\mu}_{i, T_{i}(n-1)}+\sqrt{\frac{\alpha_{j}^{R} \ln n}{2 T_{i}(n-1)}} \geq \hat{\mu}_{i', T_{i'}(n-1)}+\sqrt{\frac{\alpha_{j}^{R} \ln n}{2 T_{i'}(n-1)}}$$
    $$\Rightarrow \sqrt{\frac{\alpha_{j}^{R} \ln n}{2 T_{i'}(n-1)}} \leq 1+\sqrt{\frac{\alpha_{j}^{R} K \ln n}{2(n-1)}}$$
    $$\Rightarrow \alpha_{j}^{R} \ln n\left(1+\sqrt{\frac{\alpha_{j}^{R} K \ln n}{2(n-1)}}\right)^{-2} \leq 2 T_{i'}(n-1)$$
    $$\Rightarrow T_{i'}(n-1) \geq \frac{\alpha_{j}^{R}}{4} \ln n > \frac{1}{2} \ln n \quad \text { when } n>>1$$
    So $\exists C_{\varepsilon, \delta}>0,  \forall n>C_{\varepsilon, \delta}$ satisfies $\forall 1 \leq i \leq K$, $T_{i}(n-1)>N_{\varepsilon, \delta}$. \\
    So $\forall \varepsilon > 0,  \exists C_{\varepsilon, \delta}>0$ s.t. $\forall 1 \leq i \leq K, \left| \hat{\mu}_{i, T_{i}(n)}-\mu_{i} \right|<\varepsilon$ if $n>C_{\varepsilon, \delta}$. \\
    Then from the result of Theorem \ref{thm:122} we get the conclusion.
\end{proof}

\section{Experiment Details of Section \ref{sc:exp}}
\label{app:exp}

\subsection{Regret change after increasing parameters of UCB strategies}
\label{app:exp1}
We use $U(a,b)$ to denote a uniform distribution over $[a,b]$. To randomly generate an arm set, we generate $K \sim U(2,30)$, then $p_1, p_2, \ldots, p_K \stackrel{i.i.d}{\sim} U(0,1)$. To randomly generate a hierarchical UCB structure, we generate the number of layers $R \sim U(1,9)$, the parameter of top stragy $\beta \sim U(2,10)$, the number of experts at each layer $L_k \sim U(2,10), \forall 1 \leq k \leq R$, and the parameter of each expert $\alpha_j^k \sim U(2,10), \forall 1 \leq j \leq L_k, \forall 1 \leq k \leq R$.

Then we select a kind of distribution. For Beta distribution, we generate $a_i \sim U(2,100), b_i \sim U(2,100), \forall 1 \leq i \leq K$. For Binomial distribution, we generate $n_i \sim Discrete\_Uniform(2, 3, \cdots, 30), \forall 1 \leq i \leq K$. Then the rule of arms will be: deterministic: $X_i = p_i$. Bernoulli: $X_i \sim Bernoulli(p_i)$. Beta: $X_i \sim Beta(a_i, b_i)$. Binomial: $X_i \sim \frac{1}{n_i} Binomial(n_i, p_i)$. $\forall 1 \leq i \leq K$.

In each process, we choose a distribution and randomly generate the arm set and hierarchical UCB structure. Then we set the number of total rounds $T = 10000$ and let the experts start to choose. For the given arm set and hierarchical UCB structure, we repeat the procedure for $N=100$ times and record the average total regret at round $T$ (For deterministic, $N=1$), it will be the estimation of $R_1$. Then we finish the first half of this process.

Next, we maintain the arm set and change the parameters in hierarchical UCB structure: $\alpha_j^k \leftarrow 1000000$ if $\alpha_j^k \neq \underset{1 \leq l \leq L_k}{\min} \{\alpha_l^k\}, \forall 1 \leq j \leq L_k, \forall 1 \leq k \leq R$. Then we repeat the procedure for $N=100$ times and record the average total regret at round $T$, it will be the estimation of $R_2$. Then we finish this process.

The parameters of Figure \ref{fig:a} - \ref{fig:f} are recorded in Table \ref{tb:figure1-6}.

\begin{table}[H]
    \centering
    \text{Figure \ref{fig:a}} \\
    \begin{tabular}{|c|c|c|c|}
        \hline
        Value of $i$ & 1 & 2 & 3 \\
        \hline
        $\beta$ & 3.41 & & \\
        \hline
        $\alpha_i^1$ & 6.45 & 7.17 & 4.29 \\
        \hline
        $\mu_i$ & 0.12 & 0.51 & 0.72 \\
        \hline
    \end{tabular}
\end{table}

\begin{table}[H]
    \centering
    \text{Figure \ref{fig:b}} \\
    \begin{tabular}{|c|c|c|c|c|c|}
        \hline
        Value of $i$ & 1 & 2 & 3 & 4 & 5 \\
        \hline
        $\beta$ & 2.18 & & & & \\
        \hline
        $\alpha_i^1$ & 5.08 & 4.90 & 3.80 & 6.72 & \\
        \hline
        $\alpha_i^2$ & 8.86 & 9.73 & 9.33 & & \\
        \hline
        $\alpha_i^3$ & 9.13 & 6.93 & 7.45 & 2.87 & 6.07 \\
        \hline
        $\mu_i$ & 0.09 & 0.35 & 0.26 & 0.44 & \\
        \hline
    \end{tabular}
\end{table}

\begin{table}[H]
    \centering
    \text{Figure \ref{fig:c}} \\
    \begin{tabular}{|c|c|c|c|c|c|}
        \hline
        Value of $i$ & 1 & 2 & 3 & 4 & 5 \\
        \hline
        $\beta$ & 6.44 & & & & \\
        \hline
        $\alpha_i^1$ & 2.88 & 9.22 & 9.25 & 6.72 & \\
        \hline
        $\alpha_i^2$ & 2.48 & 6.60 & & & \\
        \hline
        $\alpha_i^3$ & 5.06 & 3.50 & 9.86 & 6.00 & 3.78 \\
        \hline
        $a_i$ & 46.97 & 84.75 & & & \\
        \hline
        $b_i$ & 8.71 & 68.52 & & & \\
        \hline
    \end{tabular}
\end{table}

\begin{table}[H]
    \centering
    \text{Figure \ref{fig:d}} \\
    \begin{tabular}{|c|c|c|c|c|c|}
        \hline
        Value of $i$ & 1 & 2 & 3 & 4 & 5 \\
        \hline
        $\beta$ & 6.85 & & & & \\
        \hline
        $\alpha_i^1$ & 2.73 & 3.11 & 7.42 & 5.24 & \\
        \hline
        $\alpha_i^2$ & 6.84 & 5.14 & & & \\
        \hline
        $n_i$ & 7 & 19 & 9 & 15 & 25 \\
        \hline
        $\mu_i$ & 0.96 & 0.31 & 0.20 & 0.23 & 0.32 \\
        \hline
    \end{tabular}
\end{table}

\begin{table}[H]
    \centering
    \text{Figure \ref{fig:e}} \\
    \begin{tabular}{|c|c|c|}
        \hline
        Value of $i$ & 1 & 2 \\
        \hline
        $\beta$ & 6.51 & \\
        \hline
        $\alpha_i^1$ & 5.69 & 7.08 \\
        \hline
        $\mu_i$ & 0.36 & 0.70 \\
        \hline
    \end{tabular}
\end{table}

\begin{table}[H]
    \centering
    \text{Figure \ref{fig:f}} \\
    \begin{tabular}{|c|c|c|c|c|c|}
        \hline
        Value of $i$ & 1 & 2 & 3 & 4 & 5 \\
        \hline
        $\beta$ & 7.60 & & & & \\
        \hline
        $\alpha_i^1$ & 2.61 & 9.72 & & & \\
        \hline
        $\alpha_i^2$ & 2.51 & 8.13 & 3.37 & 7.31 & 7.69 \\
        \hline
        $\mu_i$ & 0.79 & 0.42 & 0.01 & 0.72 & \\
        \hline
    \end{tabular}
    \caption{Parameters of Figure \ref{fig:a} - \ref{fig:f}} \label{tb:figure1-6}
\end{table}

\subsection{Selection ranges of different experts in hierarchical UCB structure}
In this experiment, we use deterministic distribution and randomly run processes in the same way as the first half of one process in Experiment \ref{app:exp1}.

\subsection{Reasonable number of experts}
In this experiment, we first randomly generate an arm set using the way in Experiment \ref{app:exp1}. Then we assume a prior distribution of the parameters in hierarchical UCB structure: $U(2,10)$. For $K \leq 500$, we calculate $M_i$ and $N_i$ for $1 \leq i \leq 100$.

To calculate $M_i$, we firstly calculate $\mathbb{E}[\beta] = 6, \mathbb{E}[\alpha_1^1] = \frac{8}{i+1} + 2$. Here $\mathbb{E}[\alpha_1^1]$ is the expectation of the minimum of $i$ i.i.d random variables which obey $U(2,10)$. Then we can calculate the expectation of $i^*$ and the bound in Theorem \ref{thm:final}.

To calculate $N_i$, we fix $L_1 = i$ and randomly generate $N=100$ hierarchical UCB structures using the way in Experiment \ref{app:exp1}. For each hierarchical UCB structure, we use the same method and parameters (Distribution of arms is assumed as deterministic) as Experiment \ref{app:exp1} to estimate the total regret. Then the average of these $N$ total regrets is the estimation of $N_i$.

\end{document}